\documentclass{article}
\usepackage{fullpage}
\usepackage{lmodern}
\usepackage[T1]{fontenc}

\usepackage{graphicx}
\usepackage{subfigure} 

\usepackage{algorithm}
\usepackage{algorithmic}

\usepackage{hyperref,xspace,amsmath,amsthm,amssymb,mathtools,xcolor,paralist,caption}

\def\etal{{\em et al.}\xspace}

        {\hspace*{\fill}\par}

        {\hspace*{\fill}$\Box$\par}

\renewenvironment{proof}{\vspace{-0.05in}\noindent{\bf Proof:}}%
        {\hspace*{\fill}$\Box$\par}
\newenvironment{proofof}[1]{\smallskip\noindent{\bf Proof of #1:}}%
        {\hspace*{\fill}$\Box$\par}
        {\hspace*{\fill}$\Box$\par}

\newtheorem{theorem}{Theorem}[section]
\newtheorem{lemma}[theorem]{Lemma}

\newtheorem{claim}[theorem]{Claim}

\theoremstyle{definition}
\newtheorem{definition}[theorem]{Definition}

\newcommand{\eps}{\varepsilon}
\newcommand{\E}{\mathbb{E}}

\newcommand{\R}{{\mathbb{R}}}
\newcommand{\supp}{\text{supp}}

\newcommand{\Fm}{F_{\max}}
\def\set#1{\left\{#1\right\}}
\newcommand{\argmin}{{\operatorname{argmin}}}

\def\final{} 
\ifdefined\final
    \newcommand{\hnote}[1]{}
    \newcommand{\lnote}[1]{}
    \newcommand{\anote}[1]{}
\else
\newcommand{\hnote}[1]{\begingroup\color{blue!50!black}\em (Huy: #1)\endgroup}
\newcommand{\anote}[1]{\begingroup\color{green!50!black}\em (Alina: #1)\endgroup}
\newcommand{\lnote}[1]{\begingroup\color{purple!50!black}\em (Laci: #1)\endgroup}
\fi

\title{Decomposable Submodular Function Minimization:\\ Discrete and Continuous}
\author{
Alina Ene\thanks{Department of Computer Science, Boston University, {\tt aene@bu.edu}}
\and
Huy L. Nguy\~{\^{e}}n\thanks{College of Computer and Information Science, Northeastern University, {\tt hu.nguyen@northeastern.edu}}
\and
L\'{a}szl\'{o} A. V\'{e}gh\thanks{Department of Mathematics, London School of Economics, {\tt L.Vegh@lse.ac.uk}}
}

\begin{document} 
\maketitle

\begin{abstract} 
This paper investigates connections between discrete and continuous
approaches for decomposable submodular function minimization. We
provide improved running time estimates for the state-of-the-art
continuous algorithms for the problem using combinatorial
arguments. We also provide a systematic experimental comparison of the two
types of methods, based on a clear distinction between level-0 and
level-1 algorithms.
\end{abstract} 

\section{Introduction}
Submodular functions arise in a wide range of applications: graph theory, optimization, economics, game theory, to name a few. 
A function $f: 2^V\to \mathbb{R}$ on a ground set $V$ is {\em submodular} if
$f(X)+f(Y)\geq f(X\cap Y)+f(X\cup Y)$ for all sets $X,Y\subseteq
V$. Submodularity can also be interpreted as a decreasing marginals property.

There has been significant interest in submodular optimization in the machine learning and computer vision communities. The {\em submodular function minimization} (SFM) problem arises
in problems in image segmentation or MAP inference tasks in Markov Random Fields. Landmark results in combinatorial optimization give polynomial-time exact algorithms for SFM. However, the high-degree polynomial dependence in the running time is prohibitive for large-scale problem instances. The main objective in this context is to develop fast and scalable SFM algorithms.

Instead of minimizing arbitrary submodular functions, several recent papers aim to exploit special structural properties of submodular functions arising in practical applications. A popular model is {\em decomposable submodular functions}: these can be written as sums of several ``simple'' submodular functions defined on small supports.

Some definitions are in order. Let $f:2^V\to\mathbb{R}$ be a submodular function, and let $n:=|V|$. We can assume w.l.o.g. that $f(\emptyset)=0$. We are interested in solving the {\em submodular function minimization problem}
\begin{equation}\label{prob:sfm}
\min_{S\subseteq V} f(S).\tag{SFM}
\end{equation}
The base polytope of a submodular function is defined as
\begin{align*}
  B(f) := \{ x \in \mathbb{R}^V \colon x(S) \leq f(S) \; \forall S \subseteq V, x(V) = f(V) \}.
\end{align*}
One can optimize linear functions over $B(f)$ using the greedy algorithm.
The problem \eqref{prob:sfm} can be reduced to finding the minimum norm point of the base polytope $B(f)$ \cite{fujishige80}.
\begin{equation}\label{prob:fuji}
\min\left\{ \frac12 \|y\|_2^2\colon y \in B(f)\right\}.\tag{Min-Norm}
\end{equation}
This reduction is the starting point of convex optimization approaches for \eqref{prob:sfm}.
We refer the reader to Sections 44--45 in \cite{Schrijver03} for concepts and results in submodular optimization, and  to \cite{Bach-monograph} on machine learning applications.

We assume that $f$ is given in the decomposition 
\[f(S) = \sum_{i = 1}^r f_i(S), \]
where each $f_i: 2^{V} \rightarrow \mathbb{R}$ is a submodular function. Such functions are called \emph{decomposable} or
\emph{Sum-of-Submodular (SoS)} in the literature. In this paper, we will use the abbreviation DSFM.

For each $i\in [r]$, the function $f_i$ has an effective support $C_i$ such that  $f_i(S)=f_i(S\cap C_i)$ for every $S\subseteq V$. For each $i\in [r]$, we assume that two oracles are provided: {\em (i)} a value oracle that returns $f_i(S)$ for any set $S\subseteq V$ in time $\mathrm{EO}_i$; and {\em (ii)} a quadratic minimization oracle ${\cal O}_i(w)$.
For any input vector $w \in \mathbb{R}^n$, this oracle returns an optimal solution to \eqref{prob:fuji} for the function $f_i+w$, or equivalently, an optimal solution to  $\min_{y \in B(f_i)} \|y + w\|^2_2$. We let $\Theta_i$ denote the running time of a single call to the oracle ${\cal O}_i$,  $\Theta_{\max}:=\max_{i\in[r]}\Theta_i$ denote the maximum time of an oracle call, $\Theta_{\mathrm{avg}} := {1 \over r} \sum_{i \in [r]} \Theta_i$ denote the average time of an oracle call.\footnote{For flow-type algorithms for DSFM, a slightly weaker oracle assumption suffices, returning a minimizer of $\min_{S \subseteq C_i} f_i(S) + w(S)$ for any given $w \in \mathbb{R}^{C_i}$. This oracle and the quadratic minimization oracle are reducible to each other:  the former reduces to a single call to the latter, and one can implement the latter using $O(|C_i|)$ calls to the former (see e.g. \cite{Bach-monograph}).} We let $F_{i,\max} := \max_{S \subseteq V} |f_i(S)|$, $\Fm:= \max_{S \subseteq V} |f(S)|$ denote the maximum function values.

Decomposable SFM thus requires algorithms on two levels. The \emph{level-0} algorithms
are the subroutines used to evaluate the oracles  ${\cal O}_i$ for every $i\in [r]$. The \emph{level-1} algorithm minimizes the function $f$ using the level-0 algorithms as black boxes.

\subsection{Prior work}
SFM has had a long history in combinatorial optimization since the early 1970s, following the influential work of Edmonds \cite{Edmonds70}. The first polynomial-time algorithm was obtained via the ellipsoid method \cite{Grotschel1981}; recent work presented substantial improvements using this approach \cite{LeeSW15}.
Substantial work focused on designing strongly polynomial combinatorial algorithms \cite{Schrijver00,iwata2001,fleischer03,Iwata03,Orlin09,IwataO09}. 
Still, designing practical algorithms for SFM that can be applied to large-scale problem instances remains an open problem.

Let us now turn to decomposable SFM.
Previous work mainly focused on level-1 algorithms. These can be classified as \emph{discrete} and \emph{continuous} optimization methods. The discrete approach builds on techniques of classical discrete algorithms for  network flows and for submodular flows. Kolmogorov \cite{kolmogorov12} showed that the problem can be reduced to submodular flow maximization, and also presented a more efficient augmenting path algorithm. Subsequent discrete approaches were given in \cite{Arora12,Fix13,Fix14}.
Continuous approaches start with convex programming formulation \eqref{prob:fuji}. Gradient methods were 
applied for the decomposable setting in \cite{StobbeK10,NishiharaJJ14,EneN15}. 

Less attention has been given to the level-0 algorithms. Some papers mainly focus on theoretical guarantees on the running time of level-1 algorithms, and treat the level-0 subroutines as black-boxes (e.g. \cite{kolmogorov12,NishiharaJJ14,EneN15}). In other papers (e.g. \cite{StobbeK10,JegelkaBS13}), the model is restricted to functions $f_i$ of a simple specific type that are easy to minimize. An alternative assumption is that all $C_i$'s are small, of size at most $k$; and thus these oracles can be evaluated by exhaustive search, in $2^k$ value oracle calls (e.g. \cite{Arora12,Fix13}). 

Shanu \etal \cite{Shanu16} use a block coordinate descent method for level-1, and allow arbitrary functions $f_i$. The oracles are evaluated via the Fujishige-Wolfe minimum norm point algorithm \cite{Fujishige11, Wolfe76} for level-0. 

\subsection{Our contributions}
Our paper establishes connections between discrete and continuous methods for decomposable SFM, as well as provides a systematic experimental comparison of these approaches.
Our main theoretical contribution improves the worst-case complexity bound of the most recent continuous optimization methods \cite{NishiharaJJ14,EneN15} by a factor of $r$, the number of functions in the decomposition. This is achieved by improving the bounds on the relevant condition numbers. Our proof exploits ideas from the discrete optimization approach. This provides not only better, but also considerably simpler arguments than the algebraic proof in \cite{NishiharaJJ14}. 

The guiding principle of our experimental work is the clean conceptual distinction between the level-0 and level-1 algorithms. Previous experimental studies considered the level-0 and level-1 algorithms as a single ``package''. For example, Shanu \etal \cite{Shanu16} compare the performance of their \emph{SoS Min-Norm algorithm} to the continuous approach of Jegelka \etal \cite{JegelkaBS13} and the combinatorial approach of Arora \etal \cite{Arora12}. However, these implementations are difficult to compare since they  use three different level-0 algorithms: Fujishige-Wolfe in SoS Min-Norm, a general QP solver for the algorithm of 
\cite{JegelkaBS13}, and exhaustive search for \cite{Arora12}. For potentials of large support, Fujishige-Wolfe outperforms these other level-0 subroutines, hence  the algorithms in \cite{JegelkaBS13,Arora12} could have compared more favorably using the same Fujishige-Wolfe subroutine.

In our experimental setup, we compare level-1 algorithms by using the same level-0 subroutines. We compare the state-of-the-art continuous and discrete  algorithms: RCDM and ACDM from \cite{EneN15} with Submodular IBFS from \cite{Fix13}. We consider multiple options for the level-0 subroutines. For certain potential types, we use tailored subroutines exploiting the specific form of the problem. We also consider a variant of the Fujishige-Wolfe algorithm as a subroutine applicable for arbitrary potentials. 
Our experimental results reveal the following tradeoff. Discrete algorithms on level-1 require more calls to the level-0 oracle, but less overhead computation. Hence using algorithms such as IBFS on level-1 can be significantly faster than gradient descent as long as the potentials have fairly small supports. However, as the size of the potentials grow, or we do need to work with a generic level-0 algorithm, the better choice is using gradient methods. Gradient methods can perform better for larger potentials also due to weaker requirements on the level-0 subroutines: approximate level-0 subroutines suffice for them, whereas discrete algorithms require exact optimal solutions on level-0. 

{\bf Paper outline.} The rest of the paper is structured as follows.
Section~\ref{sec:flow} describes the level-1 algorithmic framework for DSFM that is based on network flows, and outlines the IBFS algorithm. Section~\ref{sec:conv-opt} describes the level-1 algorithmic framework for DSFM that is based on convex optimization, and outlines the gradient descent algorithms. Section~\ref{sec:kappa} gives improved convergence guarantees for the gradient descent algorithms outlined in Section~\ref{sec:conv-opt}. Section~\ref{sec:level-0} discusses the different types of level-0 algorithms and how they can be used together with the level-1 frameworks. Section~\ref{sec:experiments} presents our experimental results.

\section{Level-1 algorithms based on network flow}
\label{sec:flow}

In this section, we outline a level-1 algorithmic framework for DSFM that is based on a combinatorial framework first studied in \cite{FujishigeZhang92}.\footnote{The framework was introduced in a slightly different context, for the submodular intersection problem. The dual of this problem is minimizing a submodular function of the form $f=f_1+f_2$, with access to oracles minimizing $f_1$ and $f_2$.} 

For a decomposable function $f$, every $x\in B(f)$ can be written as $x = \sum_{i = 1}^r x_i$, where
$\supp(x_i)\subseteq C_i$ and $x_i\in B(f_i)$ (see e.g. Theorem 44.6 in \cite{Schrijver03}). A natural algorithmic approach is to maintain an $x\in B(f)$ in such a representation, and iteratively update it using the combinatorial framework described below. 
DSFM can be casted as a maximum network flow instance in a network that is suitably defined based on the current point $x$. This  can be viewed as an analogue of the residual graph in the maxflow/mincut setting, and it is precisely the residual graph if the DSFM instance was a mincut instance.

{\bf The auxiliary graph.} For an $x\in B(f)$ of the form $x=\sum_{i = 1}^r x_i$, we construct the following directed auxiliary graph $G = (V, E)$, with $E = \bigcup_{i = 1}^r E_i$ and capacities $c:E\to \R_+$. The arc sets $E_i$ are complete directed graphs (cliques) on $C_i$, and for an arc $(u,v)\in E_i$, we define $c(u,v):=\min\{ f_i(S)-x_i(S)\colon S\subseteq C_i, u\in S, v\notin S\}$. This is the maximum value $\varepsilon$ such that $x'_i\in B(f_i)$, where $x'_i(u)=x_i(u)+\varepsilon$, $x'_i(v)=x_i(v)-\varepsilon$, $x'_i(z)=x_i(z)$ for $z\notin\{u,v\}$. 

Let 
 $N := \{ v \in V \colon x(v) < 0 \}$ and $P := \{ v \in V \colon x(v) > 0\}$. 
The algorithm aims to improve the current $x$ by updating along shortest directed paths from $N$ to $P$ with positive capacity; there are several ways to update the solution, and we discuss specific approaches later in the section. If there exists no such directed path, then we let $S$ denote the set reachable from $N$ on directed paths
with positive capacity; thus, $S\cap P=\emptyset$. It is easy to show that $S$ is a minimizer of the function $f$. 

Updating along a shortest path ${\cal Q}$ from $N$ to $P$ amounts to the following. Let $\varepsilon$ denote the minimum capacity of an arc on $\cal Q$. If $(u,v)\in {\cal Q}\cap E_i$, then we increase $x_i(u)$ by $\varepsilon$ and decrease $x_i(v)$ by $\varepsilon$. The crucial technical claim \cite{FujishigeZhang92} is the following. Let $d(u)$ denote the shortest path distance of positive capacity arcs from $u$ to the set $P$. Then, an update along a shortest directed path from $N$ to $P$ results in a feasible $x\in B(f)$, and further, all distance labels $d(u)$ are non-decreasing.

{\bf Level-1 algorithms based on the network flow approach.} Using this auxiliary graph, and updating on shortest augmenting paths, one can generalize several maximum flow algorithms to a level-1 algorithm of DSFM.
These algorithms include: the Edmonds-Karp-Dinitz maximum flow algorithm, the preflow-push algorithm \cite{GoldbergTarjan88}, the incremental breadth first search algorithm (IBFS) \cite{goldberg11}, and the excesses incremental breadth first search algorithm \cite{goldberg15}. Our experiments will use an implementation of IBFS, following~\cite{Fix13}. 

{\bf Submodular incremental breadth first search (IBFS).} Fix \etal
\cite{Fix13} adapt the IBFS algorithm to the above
described submodular framework  using the above
mentioned claims by Fujishige \& Zhang \cite{FujishigeZhang92}. IBFS is an augmenting path algorithm
for the maximum flow problem. It identifies a shortest path from
the source set $N$ to the sink set $P$ via 
growing shortest path trees simultaneously forwards from $N$ and
backwards from $P$. 

The submodular IBFS algorithm provides us with a level-1 algorithm for DSFM. Each step of the algorithm involves determining the capacity of an arc in the auxiliary graph; as we explain in Section~\ref{sec:level-0}, each of these capacities can be computed using a single call to a level-0 subroutine ${\cal O}_i$.

By combining the level-1 IBFS algorithm with appropriate level-0 subroutines, we obtain an algorithm for DSFM whose running time can be upper bounded as follows. On a directed graph with $n$ nodes and $m$ arcs, IBFS runs in time $O(n^2m)$. In the DSFM setting, we have $m=O(\sum_{i\in [r]} |C_i|^2 )$. Every step  involves determining an auxiliary capacity, which can be implemented using a single call to a level-0 subroutine ${\cal O}_i$ (see Section~\ref{sec:level-0}); the maximum time of such an oracle call is $\Theta_{\max}$. Hence, the running time bound for submodular IBFS can be given as $O(n^2\Theta_{\max} \sum_{i\in [r]}|C_i|^2)$. If all $C_i$'s are small, $O(1)$, then this gives $O(n^2r\Theta_{\max})$.

\section{Level-1 algorithms based on convex optimization}
\label{sec:conv-opt}
In this section, we outline the level-1 algorithms for DSFM that are based on gradient descent. 
Recall the convex quadratic program (\ref{prob:fuji}) from the Introduction. This program has a unique optimal solution $s^*$; the set $S=\{v \in V \colon s^*(v) < 0\}$ is the unique smallest minimizer to \eqref{prob:sfm}. We will refer to this optimal solution $s^*$ throughout the section. 

In the DSFM setting, one can write (\ref{prob:fuji}) in multiple equivalent forms \cite{JegelkaBS13}. For the first formulation, we let $\mathcal{P} := \prod_{i = 1}^r B(f_i)\subseteq \R^{rn},$ and let $A\in \R^{n\times (rn)}$ denote the following matrix:
	$$A := \underbrace{[I_n I_n \dots I_n]}_\text{$r$ times}.$$
Note that, for every $y \in \mathcal{P}$, $Ay = \sum_{i = 1}^r y_i$, where $y_i$ is the $i$-th block of $y$, and thus $Ay \in B(f)$. The problem \eqref{prob:fuji} can be reformulated for DSFM as follow.
\begin{equation}
\label{Prox-DSFM}
	\min\left\{ {1 \over 2} \left\|Ay \right\|^2_2 \colon \tag{Prox-DSFM} y\in \mathcal{P}\right\}.
\end{equation}
The second formulation is the following. Let us define the subspace $\mathcal{A} := \{ a \in \mathbb{R}^{nr} \colon Aa= 0\}$, and minimize its distance from $\mathcal{P}$:
\begin{equation}
\label{Best-Approx}
	\min\left\{ \|a - y\|^2_2 \tag{Best-Approx} \colon a \in \mathcal{A}, y \in \mathcal{P}\right\}.
\end{equation}
The set of optimal solutions for both formulations (\ref{Prox-DSFM}) and (\ref{Best-Approx}) is the set 
$\mathcal{E} := \{y \in \mathcal{P} \colon Ay = s^*\}$,
 where $s^*$ is the optimum of (\ref{prob:fuji}). We note that, even though the set of solutions to (\ref{Best-Approx}) are pairs of points $(a, y) \in \mathcal{A} \times \mathcal{P}$, the optimal solutions are uniquely determined by $y \in \mathcal{P}$, since the corresponding $a$ is the projection of $y$ to   $\mathcal{A}$.

\begin{lemma}[{\cite{JegelkaBS13}, Lemma 2}]
\label{lem:decomp-opt-set}
 The set $\mathcal{E}$ is non-empty and it coincides with the set of optimal solutions of (\ref{Prox-DSFM}) and (\ref{Best-Approx}).
\end{lemma}

{\bf Gradient methods.} The gradient descent algorithms of \cite{NishiharaJJ14,EneN15} provide level-1 algorithms for DSFM. In the following, we provide a brief overview of these algorithms and we refer the reader to the respective papers for more details. 

{\bf The alternating projections algorithm.} Nishihara \etal minimize (\ref{Best-Approx}) using \emph{alternating projections} \cite{NishiharaJJ14}. The algorithm starts with a point $a_0 \in \mathcal{A}$ and it iteratively constructs a sequence $\set{(a^{(k)}, x^{(k)})}_{k \geq 0}$ by projecting onto $\mathcal{A}$ and $\mathcal{P}$: $x^{(k)} = \argmin_{x \in \mathcal{P}}\|a^{(k)} - x\|_2$, $a^{(k + 1)} = \argmin_{a \in \mathcal{A}}\|a - x^{(k)}\|_2$.

{\bf Random coordinate descent algorithms.} Ene and Nguyen minimize (\ref{Prox-DSFM}) using \emph{random coordinate descent} \cite{EneN15}. The RCDM algorithm adapts the random coordinate descent algorithm of Nesterov \cite{Nesterov10} to (\ref{Prox-DSFM}). In each iteration, the algorithm samples a block $i \in [r]$ uniformly at random and it updates $x_i$ via a standard gradient descent step for smooth functions. ACDM, the accelerated version of the algorithm, presents a further enhancement using techniques from  \cite{FR13}.

{\bf Rate of convergence.} The algorithms mentioned above enjoy a \emph{linear convergence rate} despite the fact that the objective functions of (\ref{Best-Approx}) and (\ref{Prox-DSFM}) are not strongly convex. Instead, these works show that there are certain parameters that one can associate with the objective functions such that the convergence is at the rate $(1 - \alpha)^k$, where $\alpha \in (0, 1)$ is a quantity that depends on the appropriate parameter. Let us now precisely define these parameters and state the convergence guarantees as a function of these parameters.

 Let $\mathcal{A}'$ be the affine subspace 
$\mathcal{A}':= \{ a \in \mathbb{R}^{nr} \colon Aa = s^*\}.$ 
Note that $\mathcal{E} = \mathcal{P} \cap \mathcal{A}'$.
For $y\in \R^{nr}$ and a closed set $K\subseteq \R^{nr}$, we let $d(y,K)=\min\set{\|y-z\|_2 \colon z\in K}$ denote the distance between $y$ and $K$.
 The relevant parameter for the Alternating Projections algorithm is defined as follows.
\begin{definition}[\cite{NishiharaJJ14}]
For every $y\in (\mathcal{P} \cup \mathcal{A}') \setminus {\cal E}$, let
	\begin{align*}
		\kappa(y) &\coloneqq \frac{d(y, {\cal E})}{\max\set{d(y, \mathcal{P}), d(y, \mathcal{A}')}},\quad \mbox{ and }\\
		\kappa_* &\coloneqq \sup\left\{ \kappa(y) \colon y \in (\mathcal{P} \cup \mathcal{A}') \setminus {\cal E} \right\}.
	\end{align*}
\end{definition}

The relevant parameter for the random coordinate descent algorithms is the following.

\begin{definition}[\cite{EneN15}]
	For every $y\in\mathcal{P}$, let $y^* \coloneqq \argmin_{p} \{\|p-y\|_2 \colon Ap=s^*\}$ be the optimal solution to (\ref{Prox-DSFM}) that is closest to $y$. We say that the objective function ${1 \over 2} \|Ay\|^2_2$ of (\ref{Prox-DSFM}) is \emph{restricted $\ell$-strongly convex} if, for all $y \in \mathcal{P}$, we have
	\begin{align*}
		&\|A(y - y^*)\|^2_2 \geq \ell \|y - y^*\|^2_2,\quad \mbox{ and }\\
		\ell_* \coloneqq \sup&\left\{\ell \colon {1 \over 2} \|Ay\|^2_2 \text{ is restricted $\ell$-strongly convex}\right\}.
	\end{align*}
\end{definition}

The running time dependence of the algorithms on these parameters is given in the following theorems.
\begin{theorem}[\cite{NishiharaJJ14}]
	Let $(a^{(0)}, x^{(0)} =  \argmin_{x \in \mathcal{P}}\|a^{(0)} - x\|_2)$ be the initial solution and let  $(a^*, x^*)$ be an optimal solution to (\ref{Best-Approx}). The alternating projection algorithm produces in
		\[k = \Theta\left(\kappa_*^2  \ln\left( {\|x^{(0)} - x^*\|_2 \over \epsilon} \right) \right) \]
	iterations a pair of points $a^{(k)} \in \mathcal{A}$ and $x^{(k)} \in \mathcal{P}$ that is $\epsilon$-optimal, i.e.,
		\[\|a^{(k)} - x^{(k)}\|_2^2 \leq \|a^* - x^*\|_2^2 + \eps. \]
\end{theorem}

\begin{theorem}[\cite{EneN15}]
	Let $x^{(0)} \in \mathcal{P}$ be the initial solution and let $x^*$ be an optimal solution to (\ref{Prox-DSFM}) that minimizes $\|x^{(0)} - x^*\|_2$. The random coordinate descent algorithm produces in
	\[ k =  \Theta\left(  {r \over \ell_*} \ln\left( {\|x^{(0)}  - x^* \|_2 \over \epsilon} \right)\right)\]
	iterations a solution $x^{(k)}$ that is $\epsilon$-optimal in expectation, i.e., $\E\left[{1 \over 2} \|A x^{(k)} \|^2_2 \right] \leq {1 \over 2} \|Ax^*\|^2_2 + \epsilon$.
	
	The accelerated coordinate descent algorithm produces in 
	\[k = \Theta \left (r \sqrt{{1 \over \ell_*}} \ln\left( { \|x^{(0)} - x^* \|_2 \over \epsilon} \right) \right)  \]
	iterations (specifically, $\Theta\left( \ln\left( { \|x^{(0)} - x^*\|_2 \over \epsilon} \right) \right)$ epochs with $\Theta\left(r \sqrt{{1 \over \ell_*}}\right)$ iterations in each epoch) a solution $x^{(k)}$ that is $\epsilon$-optimal in expectation, i.e., $\E\left[{1 \over 2} \|A x^{(k)} \|^2_2 \right] \leq {1 \over 2} \|Ax^*\|^2_2 + \epsilon$.
\end{theorem}

Nishihara \etal show that $\kappa_* \leq nr$, and a family of instances (in fact, minimum cut instances) is given for which $\kappa_* \geq \Omega(n \sqrt{r})$. Ene and Nguyen show that $\ell_* \geq {r /\kappa^2_*}$. In Theorem~\ref{thm:kappa-bound}, we close the remaining gap and show that $\kappa_* = \Theta(n \sqrt{r})$ and $\ell_* = \Theta(1 / n^2)$, and thus we obtain tight analyses for the running times of the above mentioned algorithms.

By combining the level-1 gradient descent algorithms with appropriate level-0 subroutines, we obtain algorithms for DSFM whose running times can be upper bounded as follows. Using our improved convergence guarantees, it follows that RCDM obtains in time $O\left(n^2 r\Theta_{\mathrm{avg}}\ln\left( { \|x^{(0)} - x^* \|_2 \over \epsilon}\right)\right)$ a solution that is $\varepsilon$-approximate in expectation. For ACDM, the improved time bound is $O\left(nr \Theta_{\mathrm{avg}} \ln\left( { \|x^{(0)} - x^* \|_2 \over \epsilon} \right) \right)$. We can upper bound the diameter of the base polytope by $O(\sqrt{n} F_{\max})$ \cite{Jegelka11}. For integer-valued functions, a $\varepsilon$-approximate solution can be converted to an exact optimum if $\varepsilon=O(1/n)$ \cite{Bach-monograph}.

\section{Tight convergence bounds for the continuous algorithms}
\label{sec:kappa}

In this section, we show that the combinatorial approach introduced in Section~\ref{sec:flow} can be applied to obtain better bounds on the parameters $\kappa_*$ and $\ell_*$ defined in Section~\ref{sec:conv-opt}. Besides giving a stronger bound, our proof is considerably simpler than the algebraic one using Cheeger's inequality in \cite{NishiharaJJ14}.
The key is the following lemma.
\begin{lemma}
\label{lem:decompose}
	Let $y \in \mathcal{P}$ and $s^* \in B(f)$. Then there exists a point $x \in \mathcal{P}$ such that $Ax = s^*$ and $\|x - y\|_2 \leq \frac{\sqrt{n}}{2} \|Ay - s^*\|_1$.
\end{lemma}
Before proving this lemma, we show how it can be used to derive the bounds.
\begin{theorem}\label{thm:kappa-bound}
	We have $\kappa_* \leq n \sqrt{r}/2 + 1$ and $\ell_* \geq {4 / n^2}$.
\end{theorem}
\begin{proof}
	We start with the bound on $\kappa_*$. In order to bound $\kappa_*$, we need to upper bound $\kappa(y)$ for any $y\in (\mathcal{P}\cup\mathcal{A}')\setminus \mathcal{E}$. We distinguish between two cases: $y \in \mathcal{P} \setminus\mathcal{E}$ and $y \in \mathcal{A}' \setminus\mathcal{E}$.

	\noindent{\bf Case I: $y \in \mathcal{P}\setminus\mathcal{E}$.} The denominator in the definition of $\kappa(y)$ is equal to $d(y, \mathcal{A}') = {\|Ay - s^*\|_2}/{\sqrt{r}}$. This follows since the closest point $a=(a_1,\ldots,a_r)$ to $y$ in $\mathcal{A}'$ is to set $a_i=y_i+(s^*-Ay)/r$ for each $i\in [r]$. 
Lemma~\ref{lem:decompose} gives an $x \in \mathcal{P}$ such that $Ax = s^*$ and $\|x - y\|_2 \leq \frac{\sqrt{n}}{2}  \|Ay - s^*\|_1\le \frac{n}2 \|Ay - s^*\|_2$. Since $Ax = s^*$, we have $x \in {\cal E}$ and thus the numerator of $\kappa(y)$ is at most $\|x - y\|_2$. Thus $\kappa(y) \leq {\|x - y\|_2 /(\|Ay - s^*\|_2/\sqrt{r})} \leq n\sqrt{r}/2$.

	\noindent{\bf Case II: $y \in \mathcal{A}'\setminus\mathcal{E}$.}
This means that $Ay = s^*$. The denominator of $\kappa(y)$ is equal to $d(y, \mathcal{P})$. For each $i \in [r]$, let $q_i \in B(f_i)$ be the point that minimizes $\|y_i - q_i\|_2$. Let $q = (q_1, \dots, q_r)\in {\cal P}$. Then $d(y, \mathcal{P}) = \|y - q\|_2$. Lemma~\ref{lem:decompose} with $q$ in place of $y$ gives a point $x\in {\cal E}$ such that $\|q - x\|_2 \leq \frac{\sqrt{n}}2 \|Aq - s^*\|_1$. We have
	 $\|Aq - s^*\|_1 = \|Aq - Ay\|_1 \leq \sum_{i = 1}^r \|q_i - y_i\|_1 = \|q - y\|_1 \leq {\sqrt{nr}} \|q - y\|_2.$
	Thus $\|q - x\|_2 \leq \frac{{n \sqrt{r}} }2\|q - y\|_2$. 
	Since $x \in {\cal E}$, we have 
	$d(y, {\cal E}) \leq \|x - y\|_2 \leq \|x - q\|_2 + \|q - y\|_2 \leq \left(1 + \frac{n \sqrt{r}}2\right) \|q - y\|_2 = \left(1 + \frac{n \sqrt{r}}2\right) d(y, \mathcal{P}).$
	Therefore $\kappa(p) \leq 1 + \frac{n \sqrt{r}}2$, as desired.

Let us now prove the bound on $\ell_*$. Let $y \in \mathcal{P}$ and let  $y^* \coloneqq \argmin_{p} \{\|p-y\|_2\colon Ap=s^*\}$. We need to verify that $\|A(y - y^*)\|^2_2 \geq {4 \over n^2} \|y - y^*\|^2_2$.
Again, we apply Lemma~\ref{lem:decompose} to obtain a point $x\in \mathcal{P}$ such that $Ax = s^*$ and  
$\|x - y\|_2^2 \leq \frac{n}4\|Ax-Ay\|_1^2 \le \frac{n^2}4 \|Ax - Ay\|_2^2$.
Since $Ax= s^*$, the definition of $y^*$ gives $\|y - y^*\|_2^2 \leq \|x - y \|_2^2$. Using that $Ax = Ay^* = s^*$, we have $\|Ax - Ay\|_2 =\|Ay-Ay^*\|_2$. The same calculation as in Case II above implies the required $\|y - y^*\|_2^2 \leq \frac{n^2}4 \|A(y - y^*)\|_2^2$.
\end{proof}

\begin{proofof}{Lemma~\ref{lem:decompose}}
We give an algorithm that transforms $y$ to a vector $x\in {\cal P}$ as in the statement through a sequence of path augmentations in the auxiliary graph defined in Section~\ref{sec:flow}.
We initialize $x=y$ and maintain $x\in {\cal P}$ (and thus $Ax\in B(f)$) throughout. We now define the set of source and sink nodes as 
$N := \{ v \in V \colon (Ax)(v) < s^*(v) \}$ and $P := \{ v \in V \colon (Ax)(v) > s^*(v)\}$. Once $N=P=\emptyset$, we have $Ax=s^*$ and terminate. Note that since $Ax,s^*\in B(f)$, we have $\sum_v (Ax)(v)=\sum_v s^*(v)=f(V)$, and therefore $N=\emptyset$ is equivalent to $P=\emptyset$.
The blocks of $x$ are denoted as $x=(x_1,x_2,\ldots,x_r)$, with $x_i\in B(f_i)$.
\begin{claim} If $N\neq\emptyset$, then there exists a directed path of positive capacity in the auxiliary graph between the sets $N$ and $P$.
\end{claim}
\begin{proof}
We say that a set $T$ is $i$-tight, if $x_i(T)=f_i(T)$. It is a simple consequence of submodularity that the intersection and union of two $i$-tight sets are also $i$-tight sets.
For every $i\in [r]$ and every $u\in V$, we define $T_i(u)$ as the unique minimal $i$-tight set containing $u$. It is easy to see that for an arc $(u,v)\in E_i$, 
$c(u,v)>0$ if and only if $v\in T_i(u)$. We note that if $u\notin C_i$, then $x(u)=f_i(\{u\})=0$ and thus $T_i(u)=\{u\}$.

Let $S$ be the set of vertices reachable from $N$ on a directed path of positive capacity in the auxiliary graph. For a contradiction, assume $S\cap P=\emptyset$. 
By the definition of $S$, we must have $T_i(u)\subseteq S$ for every $u\in S$ and every $i\in[r]$. Since the union of $i$-tight sets is also $i$-tight, we see that  $S$ is $i$-tight for every $i\in[r]$, and consequently, $x(S)=f(S)$. On the other hand, since $N \subseteq S$, $S \cap P = \emptyset$, and $N \neq \emptyset$, we have $x(S) < s^*(S)$. Since $s^* \in B(f)$, we have $x(S) < s^*(S) \leq f(S)$, which is a contradiction. We conclude that $S \cap P \neq \emptyset$.
\end{proof}

In every step of the algorithm, we take a shortest directed path ${\cal Q}$ of positive capacity from $N$ to $P$, and update $x$ along this path. That is, 
if $(u,v)\in {\cal Q}\cap E_i$, then we increase $x_i(u)$ by
$\varepsilon$ and decrease $x_i(v)$ by $\varepsilon$, where
$\varepsilon$ is the minimum capacity of an arc on
$\cal Q$. Note that this is the same as running the Edmonds-Karp-Dinitz algorithm in the submodular auxiliary graph. Using the analysis in \cite{FujishigeZhang92}, one can show that this change maintains $x\in {\cal P}$, and that the algorithm terminates in finite (in fact, strongly polynomial) time.

It remains to bound $\|x-y\|_2$. At every path update, the change in $\ell_{\infty}$-norm of $x$ is at most $\varepsilon$, and the change in $\ell_1$-norm is at most $n\varepsilon$, since the length of the path is $\le n$. At the same time, $\sum_{v\in N} (s^*(v)-(Ax)(v))$ decreases by $\varepsilon$.
Thus, $\|x-y\|_\infty\le \|Ay-s^*\|_1 /2$ and $\|x-y\|_1\le n\|Ay-s^*\|_1 /2$. Using the inequality $\|p\|_2\le \sqrt{\|p\|_1\|p\|_{\infty}}$, we obtain $\|x-y\|_2\le \frac{\sqrt{n}}2 \|Ay-s^*\|_1$, completing the proof.
\end{proofof}

\section{The level-0 algorithms}
\label{sec:level-0}

In this section, we briefly discuss the level-0 algorithms and the interface between the level-1 and level-0 algorithms.

{\bf Two-level frameworks via quadratic minimization oracles.}
Recall from the Introduction the assumption on the subroutines ${\cal O}_i(w)$ that finds the minimum norm point in $B(f_i+w)$ for the input vector $w\in\R^n$. 
The continuous methods in Section~\ref{sec:conv-opt} directly use the  subroutines ${\cal O}_i(w)$ for the alternating projection or coordinate descent steps. For the flow-based algorithms in Section~\ref{sec:flow}, the main oracle query is to find the auxiliary graph capacity $c(u,v)$ of an arc $(u,v)\in E_i$ for some $i\in[r]$. This can be easily formulated as minimizing the function $f_i+w$ for an appropriate $w$ with $\supp(w)\subseteq C_i$; the details are given in Lemma~\ref{lem:exchange-cap-oracle}. As explained at the beginning of Section~\ref{sec:conv-opt}, an optimal solution to \eqref{prob:fuji} immediately gives an optimal solution to \eqref{prob:sfm} for the same submodular function. Hence, the auxiliary graph capacity queries can be implemented via the subroutines ${\cal O}_i(w)$. 
Let us also remark that, while the functions $f_i$ are formally defined on the entire ground set $V$, their effective support is $C_i$, and thus it suffices to solve the quadratic minimization problems on the ground set $C_i$.

\begin{lemma}
\label{lem:exchange-cap-oracle}
  The capacity $c(u, v) := \min\{f_i(S) - x_i(S) \colon S \subseteq
  C_i, u \in S, v \notin S\}$ can be computed as the minimum value of
  $\min_{S\subseteq C_i} f_i(S)+w(S)$ for an appropriately chosen vector $w\in\R^n$,
  $\supp(w)\subseteq C_i$.
\end{lemma}
\begin{proof}
  We define a weight vector $w \in \mathbb{R}^n$ as follows: $w(u)
  = - (f_i(\{u\}) + 1)$; $w(v) = - (f_i(C_i) - f_i(C_i \setminus
  \{v\}) - 1)$; $w(a) = - x(a)$ for all $a \in C_i \setminus \{u,
  v\}$, and $w(a)=0$ for all $a\notin C_i$.
Let $A\subseteq C_i$ be a minimizer of $\min_{S \subseteq C_i} f_i(S)
+ w(S)$.
It suffices to show that $u \in A$ and $v \notin A$. Note that
$f_i(\{u\}) = f_i(\{u\}) - f(\emptyset)$ is the maximum marginal value
of $u$, i.e., $\max_{S} (f_i(S \cup \{u\}) - f_i(S))$. Moreover,
$f_i(C_i) - f_i(C_i \setminus \{v\})$ is the minimum marginal value of
$v$. To show $u\in A$, let us assume for a contradiction that $u\notin A$.
  \begin{align*}
    f_i(A \cup \{u\}) + w(A \cup \{u\}) &= (f_i(A) + w(A)) + (f_i(A \cup \{u\}) - f_i(A)) + w(u)\\
    &= (f_i(A) + w(A)) + (f_i(A \cup \{u\}) - f_i(A)) - f_i(\{u\}) + 1\\
    &\leq f_i(A) + w(A) - 1.
  \end{align*}
  Similarly, to show that $v\notin A$, suppose for a contradiction that $v \in A$, and consider the set $A \setminus \{v\}$. Since $f_i(C_i) - f_i(C_i \setminus \{v\}) \leq f_i(A) - f_i(A \setminus \{v\})$, we have
  \begin{align*}
    f_i(A \setminus \{v\}) + w(A \setminus \{v\}) &= (f_i(A) + w(A)) - (f_i(A) - f_i(A \setminus \{v\})) - w(v)\\
    &= (f_i(A) + w(A)) - (f_i(A) - f_i(A \setminus \{v\})) + (f_i(C_i) - f_i(C_i \setminus \{v\})) - 1\\
    &\leq f_i(A) + w(A) - 1.
  \end{align*}
  Therefore $u \in A$ and $v \notin A$, and hence $A \in \argmin\{f_i(S) - x_i(S) \colon u \in S, v \notin S\}$. 
\end{proof}

Whereas discrete and continuous algorithms require the same type of oracles, there is an important difference between the two algorithms in terms of exactness for the oracle solutions. The discrete algorithms require exact values of the auxiliary graph capacities $c(u,v)$, as they must maintain $x_i\in B(f_i)$ throughout. Thus, the oracle  must always return an optimal solution. The continuous algorithms are more robust, and return a solution with the required accuracy even if the oracle only returns an approximate solution. As discussed in Section~\ref{sec:experiments}, this difference leads to the continuous methods being applicable in settings where the combinatorial algorithms are prohibitively slow.

{\bf Level-0 algorithms.}
We now discuss specific algorithms for quadratic minimization over the base polytopes of the functions $f_i$. Several functions that arise in applications are ``simple'', meaning that there is a function-specific quadratic minimization subroutine that is very efficient. If a function-specific subroutine is not available, one can use a general-purpose submodular minimization algorithm. The works \cite{Arora12,Fix13} use a {\em brute force search} as the subroutine for each each $f_i$, whose running time is $2^{|C_i|} \mathrm{EO}_i$. However, this is applicable only for small $C_i$'s and is not suitable for our experiments where the maximum clique size is quite large. 
As a general-purpose algorithm, we used the {\em Fujishige-Wolfe minimum norm point algorithm} \cite{Fujishige11, Wolfe76}. This provides an $\varepsilon$-approximate solution in $O(|C_i| F^2_{i,\max}/\varepsilon)$ iterations, with overall running time bound $O((|C_i|^4 + |C_i|^2 \mathrm{EO}_i) F^2_{i, \max} / \varepsilon)$. \cite{Chakrabarty14}. 
The experimental running time of the Fujishige-Wolfe algorithm can be prohibitively large \cite{jegelka11fast}. As we discuss in Section~\ref{sec:experiments}, by warm-starting the algorithm and performing only a small number of iterations, we were able to use the algorithm in conjunction with the gradient descent level-1 algorithms. 

\section{Experiments}
\label{sec:experiments}


\begin{table}
\caption{Instance sizes}
\label{tb:instance-sizes}
\centering
\begin{tabular}{|c|c|c|c|c|}
\hline
{\bf image} & {\bf \# pixels} & {\bf \# edges} & {\bf \# squares}\\
\hline
bee & 273280 & 1089921 & 68160 \\
\hline
octopus & 273280 & 1089921 & 68160\\
\hline
penguin & 154401 & 615200 & 38400\\
\hline
plant & 273280 & 1089921 & 68160\\
\hline
plane & 154401 & 615200 & 38400 \\
\hline
\end{tabular}
\begin{tabular}{|c|c|c|c|}
\hline
{\bf \# regions} & \multicolumn{3}{|c|}{{\bf min, max, and average region size}}\\
\hline
50 & 298 & 299 & 298.02\\
\hline
49 & 7 & 299 & 237.306\\
\hline
50 & 5 & 299 & 279.02\\
\hline
50 & 8 & 298 & 275.22\\
\hline
50 & 10 & 299 & 291.48\\
\hline
\end{tabular}
\end{table}

\begin{table}
\caption{Minimum cut experiments}
\label{tb:mincut}
\centering
\begin{tabular}{|c|c|c|}
\hline
{\bf image} & {\bf \# functions ($r$)} & {\bf IBFS time (sec)}\\
\hline
bee & 1363201 & 1.70942\\
\hline
octopus & 1363201 & 1.09101 \\
\hline
penguin & 769601 & 0.684413 \\
\hline
plant & 1363201 & 1.30977 \\
\hline
plane & 769601 & 0.745521 \\
\hline
\end{tabular}
\begin{tabular}{|c|c|c|c|}
\hline
\multicolumn{4}{|c|}{{\bf UCDM time (sec)}}\\
\hline
{\bf \# iter = $5r$} & {\bf \# iter = $10r$} & {\bf \# iter = $100r$} & {\bf \# iter = $1000r$}\\
\hline
0.951421 & 1.6234 & 13.4594 & 134.719 \\
\hline
0.937317 & 1.6279 & 13.9887 & 137.969 \\
\hline
0.492372 & 0.836147 & 7.1069 & 70.1742\\
\hline
0.943306 & 1.63492 & 13.9559 & 137.865 \\
\hline
0.521685 & 0.850145 & 7.31664 & 71.8874 \\
\hline
\end{tabular}
\begin{tabular}{|c|c|c|c|}
\hline
\multicolumn{4}{|c|}{{\bf ACDM time (sec)}}\\
\hline
{\bf \# iter = $5r$} & {\bf \# iter = $10r$} & {\bf \# iter = $100r$} & {\bf \# iter = $1000r$}\\
\hline
1.3769 & 2.2696 & 18.4351 & 182.069 \\
\hline
1.40884 & 2.33431 & 19.0471 & 188.887 \\
\hline
0.757929 & 1.24094 & 9.99443 & 98.5717 \\
\hline
1.39893 & 2.29446 & 18.6846 & 185.274 \\
\hline
0.766455 & 1.26081 & 10.1244 & 99.0298 \\
\hline
\end{tabular}
\end{table}
\begin{table}
\caption{Small cliques experiments}
\label{tb:smallcliques}
\centering
\begin{tabular}{|c|c|c|}
\hline
{\bf image} & {\bf \# functions ($r$)} & {\bf IBFS time (sec)}\\
\hline
bee & 1431361 & 14.5125 \\
\hline
octopus & 1431361 & 12.9877 \\
\hline
penguin & 808001 & 7.58177 \\
\hline
plant & 1431361 & 13.7403 \\
\hline
plane & 808001 & 7.67518 \\
\hline
\end{tabular}
\begin{tabular}{|c|c|c|c|}
\hline
\multicolumn{4}{|c|}{{\bf RCDM time (sec)}}\\
\hline
{\bf \# iter = $5r$} & {\bf \# iter = $10r$} & {\bf \# iter = $100r$} & {\bf \# iter = $1000r$}\\
\hline
4.14091 & 7.57959 & 66.0576 & 660.496 \\
\hline
4.29358 & 7.80816 & 68.5862 & 675.23 \\
\hline
2.16441 & 4.08777 & 37.8157 & 372.733 \\
\hline
4.6404 & 8.21702 & 69.059 & 672.753 \\
\hline
2.182 & 4.12521 & 37.8602 & 373.825 \\
\hline
\end{tabular}
\begin{tabular}{|c|c|c|c|}
\hline
\multicolumn{4}{|c|}{{\bf ACDM time (sec)}}\\
\hline
{\bf \# iter = $5r$} & {\bf \# iter = $10r$} & {\bf \# iter = $100r$} & {\bf \# iter = $1000r$}\\
\hline
5.24474 & 10.0951 & 98.7737 & 932.954 \\
\hline
5.5891 & 10.7124 & 99.4081 & 924.076 \\
\hline
2.95226 & 5.71215 & 52.9766 & 512.665 \\
\hline
5.8395 & 11.0806 & 102.023 & 900.979 \\
\hline
2.95003 & 5.70771 & 53.7524 & 486.294 \\
\hline
\end{tabular}
\end{table}
\begin{table}
\caption{Large cliques experiments with potential specific quadratic minimization for the region potentials. In order to be able to run IBFS, we used smaller regions: $50$ regions with an average size between $45$ and $50$.}
\label{tb:largecliques}
\centering
\begin{tabular}{|c|c|c|}
\hline
{\bf image} & {\bf \# functions ($r$)} & {\bf IBFS time (sec)}\\
\hline
bee & 1431411 & 14.7271\\
\hline
octopus & 1431411 & 12.698 \\
\hline
penguin & 808051 & 7.51067 \\
\hline
plant & 1431411 & 13.6282 \\
\hline
plane & 808051 & 7.64527 \\
\hline
\end{tabular}
\begin{tabular}{|c|c|c|c|}
\hline
\multicolumn{4}{|c|}{{\bf RCDM time (sec)}}\\
\hline
{\bf \# iter = $5r$} & {\bf \# iter = $10r$} & {\bf \# iter = $100r$} & {\bf \# iter = $1000r$}\\
\hline
 4.29954 & 7.87555 & 67.8876 & 664.816\\
\hline
 4.18879 & 7.61576 & 66.7 & 656.71\\
\hline
 2.132 & 4.01926 & 36.9896 & 364.694\\
\hline
 4.55894 & 8.06429 & 67.72 & 659.685\\
\hline
 2.16248 & 4.0713 & 37.1917 & 366.272\\
\hline
\end{tabular}

\begin{tabular}{|c|c|c|c|}
\hline
\multicolumn{4}{|c|}{{\bf ACDM time (sec)}}\\
\hline
{\bf \# iter = $5r$} & {\bf \# iter = $10r$} & {\bf \# iter = $100r$} & {\bf \# iter = $1000r$}\\
\hline
 5.34726 & 10.3231 & 100.24 & 912.477\\
\hline
 5.44726 & 10.4446 & 96.2384 & 898.579\\
\hline
 2.90223 & 5.60117 & 51.9775 & 500.083\\
\hline
 5.72946 & 10.8512 & 99.6597 & 879.872\\
\hline
 2.89726 & 5.61102 & 52.5439 & 475.967\\
\hline
\end{tabular}
\end{table}


We evaluate the algorithms on energy minimization problems that arise
in image segmentation problems.
 We follow the standard approach and model the image segmentation task of segmenting an object from the background as finding a minimum cost $0/1$ labeling of the pixels. The total labeling cost is the sum of labeling costs corresponding to \emph{cliques}, where a clique is a set of pixels. We refer to the labeling cost functions as clique potentials.

The main focus of our experimental analysis is to compare the running times of the decomposable submodular minimization algorithms. Therefore we have chosen to use the simple hand-tuned potentials that were used in previous work~\cite{Shanu16,Arora12,StobbeK10}: the \emph{edge-based costs} defined by \cite{Arora12} and the \emph{count-based costs} defined by \cite{StobbeK10}. Specifically, we used the following clique potentials in our experiments, all of which are submodular:
\begin{compactitem}
\item {\bf Unary potentials} for each pixel. The unary potentials are derived from Gaussian Mixture Models of color features \cite{RotherKB04}.
\item {\bf Pairwise potentials} for each edge of the $8$-neighbor grid graph. Each graph edge $(i, j)$ between pixels $i$ and $j$ is assigned a weight that is a function of $\exp(- \|v_i - v_j\|^2)$, where $v_i$ is the RGB color vector of pixel $i$. The clique potential for the edge is the cut function of the edge: the cost of a labeling is equal to zero if the two pixels have the same label and it is equal to the weight of the edge otherwise.
\item {\bf Square potentials} for each $2 \times 2$ square of pixels. We view a $2 \times 2$ square as a graph on $4$ nodes connected with $4$ edges (two horizontal and two vertical edges). The cost of a labeling is the square root of the number of edges of the square that have different labels. This is the basic edge-based potential defined by \cite{Arora12}.
\item {\bf Region potentials} for a set of regions of the image. We compute a set of regions of the image using the region growing algorithm suggested by \cite{StobbeK10}. For each region $C_i$, we define a count-based clique potential as in \cite{StobbeK10,Shanu16}: for each set $S \subseteq C_i$ of pixels, $f_i(S) = |S| |C_i \setminus S|$.
\end{compactitem}
We used five image segmentation instances to evaluate the algorithms\footnote{The data is available at \url{http://melodi.ee.washington.edu/~jegelka/cc/index.html} and \url{http://research.microsoft.com/en-us/um/cambridge/projects/visionimagevideoediting/segmentation/grabcut.htm}}. Table~\ref{tb:instance-sizes} provides the sizes of the resulting instances. The experiments were carried out on a single computer with a 3.3 GHz Intel Core i5 processor and 8 GB of memory. The reported times are averaged over 10 trials.

\begin{table}
\caption{Large cliques experiments with Fujishige-Wolfe quadratic minimization algorithm for the region potentials. The Fujishige-Wolfe algorithm was run for $10$ iterations starting from the current gradient descent solution. The region sizes are given in Table~\ref{tb:instance-sizes}.}
\label{tb:largecliquesmnp}
\centering
\begin{tabular}{|c|c|c|c|}
\hline
\multicolumn{4}{|c|}{{\bf RCDM time (sec)}}\\
\hline
{\bf \# iter = $5r$} & {\bf \# iter = $10r$} & {\bf \# iter = $100r$} & {\bf \# iter = $1000r$}\\
\hline
4.4422 & 8.18077 & 69.0444 & 674.526 \\
\hline
4.30835 & 7.86231 & 68.1428 & 665.57\\
\hline
2.2724 & 4.28243 & 38.1329 & 366.549\\
\hline
4.61008 & 8.20094 & 68.8351 & 660.469\\
\hline
2.28484 & 4.30316 & 38.0435 & 366.825\\
\hline
\end{tabular}
\begin{tabular}{|c|c|c|c|}
\hline
\multicolumn{4}{|c|}{{\bf ACDM time (sec)}}\\
\hline
{\bf \# iter = $5r$} & {\bf \# iter = $10r$} & {\bf \# iter = $100r$} & {\bf \# iter = $1000r$}\\
\hline
5.29305 & 10.2853 & 103.452 & 936.613\\
\hline
5.55511 & 10.6411 & 97.955 & 901.875\\
\hline
2.95909 & 5.74585 & 54.3808 & 505.977\\
\hline
5.71402 & 10.8467 & 99.6515 & 873.694\\
\hline
2.9556 & 5.73271 & 54.0599 & 482.496\\
\hline
\end{tabular}
\end{table}

{\bf Number of iterations for the coordinate methods.} We have run the coordinate descent algorithms for $1000 r$ iterations, where $r$ is the number of functions in the decomposition. Our choice is based on the empirical results of Jegelka \etal \cite{JegelkaBS13} that showed that this number of iterations suffices to obtain good results.

{\bf Minimum cut experiments.} We evaluated the algorithms on instances containing only the unary potentials and the pairwise potentials. Table~\ref{tb:mincut} gives the running times in seconds.

{\bf Small cliques experiments.} We evaluated the algorithms on instances containing the unary potentials, the pairwise potentials, and the square potentials. Table~\ref{tb:smallcliques} gives the running times in seconds.

{\bf Large cliques experiments.} We evaluated the algorithms on instances containing all of the potentials: the unary potentials, the pairwise potentials, the square potentials, and the region potentials. For the region potentials, we used a potential-specific level-$0$ algorithm that performs quadratic minimization over the base polytope in time $O(|C_i| \log(|C_i|) + |C_i| \mathrm{EO}_i)$. Additionally, due to the slow running time of IBFS, we used smaller regions: $50$ regions with an average size between $45$ and $50$.

{\bf Large cliques experiments with Fujishige-Wolfe algorithm.} We also ran a version of the large cliques experiments with the Fujishige-Wolfe algorithm as the level-$0$ algorithm for the region potentials. The Fujishige-Wolfe algorithm was significantly slower than the potential-specific quadratic minimization algorithm and in our experiments it was prohibitive to run the Fujishige-Wolfe algorithm to near-convergence. Since the IBFS algorithm requires almost exact quadratic minimization in order to compute exchange capacities, it was prohibitive to run the IBFS algorithm with the Fujishige-Wolfe algorithm. In contrast, the coordinate descent methods can potentially make progress even if the level-$0$ solution is far from being converged.

In order to empirically evaluate this hypothesis, we made a simple but crucial change to the Fujishige-Wolfe algorithm: we \emph{warm-started} the algorithm with the current solution. Recall that the coordinate descent algorithms maintain a solution $x_i \in B(f_i)$ for each function $f_i$ in the decomposition. We warm-started the Fujishige-Wolfe algorithm with the current solution $x_i$, and we ran the algorithm for a small number of iterations. In our experiments, we ran the Fujishige-Wolfe algorithm for $10$ iterations. These changes made the level-$0$ running time considerably smaller, which made it possible to run the level-$1$ coordinate descent algorithms for as many as $1000r$ iterations. At the same time, performing $10$ iterations starting from the current solution seemed enough to provide an improvement over the current solution. Table~\ref{tb:largecliquesmnp} gives the running times.

{\bf Conclusions.} The combinatorial level-$1$ algorithms such as IBFS are exact and can be significantly faster than the gradient descent algorithms provided that the sizes of the cliques are fairly small. For instances with larger cliques, the combinatorial algorithms are no longer suitable if the only choice for the level-$0$ algorithms are generic methods such as the Fujishige-Wolfe algorithm. The experimental results suggest that in such cases, the coordinate descent methods together with a suitably modified Fujishige-Wolfe algorithm provides an approach for obtaining an approximate solution.

\newpage \clearpage
\bibliographystyle{alpha}
\bibliography{dsm}

\newcommand{\etalchar}[1]{$^{#1}$}
\begin{thebibliography}{FJMPZ13}

\bibitem[ABKM12]{Arora12}
Chetan Arora, Subhashis Banerjee, Prem Kalra, and SN~Maheshwari.
\newblock Generic cuts: An efficient algorithm for optimal inference in higher
  order mrf-map.
\newblock In {\em European Conference on Computer Vision}, pages 17--30.
  Springer, 2012.

\bibitem[Bac11]{Bach-monograph}
Francis Bach.
\newblock Learning with submodular functions: {A} convex optimization
  perspective.
\newblock {\em ArXiv preprint arXiv:1111.6453}, 2011.

\bibitem[CJK14]{Chakrabarty14}
Deeparnab Chakrabarty, Prateek Jain, and Pravesh Kothari.
\newblock Provable submodular minimization using {Wolfe's} algorithm.
\newblock In {\em Advances in Neural Information Processing Systems}, pages
  802--809, 2014.

\bibitem[Edm70]{Edmonds70}
Jack Edmonds.
\newblock Submodular functions, matroids, and certain polyhedra.
\newblock {\em Combinatorial structures and their applications}, pages 69--87,
  1970.

\bibitem[EN15]{EneN15}
A.~R. Ene and H.~L. {Nguyen}.
\newblock Random coordinate descent methods for minimizing decomposable
  submodular functions.
\newblock In {\em Proceedings of the 32nd International Conference on Machine
  Learning (ICML)}, 2015.

\bibitem[FI03]{fleischer03}
Lisa Fleischer and Satoru Iwata.
\newblock A push-relabel framework for submodular function minimization and
  applications to parametric optimization.
\newblock {\em Discrete Applied Mathematics}, 131(2):311--322, 2003.

\bibitem[FI11]{Fujishige11}
Satoru Fujishige and Shigueo Isotani.
\newblock A submodular function minimization algorithm based on the
  minimum-norm base.
\newblock {\em Pacific Journal of Optimization}, 7(1):3--17, 2011.

\bibitem[FJMPZ13]{Fix13}
Alexander Fix, Thorsten Joachims, Sung Min~Park, and Ramin Zabih.
\newblock Structured learning of sum-of-submodular higher order energy
  functions.
\newblock In {\em Proceedings of the IEEE International Conference on Computer
  Vision}, pages 3104--3111, 2013.

\bibitem[FR15]{FR13}
Olivier Fercoq and Peter Richt{\'a}rik.
\newblock Accelerated, parallel, and proximal coordinate descent.
\newblock {\em SIAM Journal on Optimization}, 25(4):1997--2023, 2015.

\bibitem[Fuj80]{fujishige80}
Satoru Fujishige.
\newblock Lexicographically optimal base of a polymatroid with respect to a
  weight vector.
\newblock {\em Mathematics of Operations Research}, 5(2):186--196, 1980.

\bibitem[FWZ14]{Fix14}
Alexander Fix, Chen Wang, and Ramin Zabih.
\newblock A primal-dual algorithm for higher-order multilabel markov random
  fields.
\newblock In {\em Proceedings of the IEEE Conference on Computer Vision and
  Pattern Recognition}, pages 1138--1145, 2014.

\bibitem[FZ92]{FujishigeZhang92}
Satoru Fujishige and Xiaodong Zhang.
\newblock New algorithms for the intersection problem of submodular systems.
\newblock {\em Japan Journal of Industrial and Applied Mathematics}, 9(3):369,
  1992.

\bibitem[GHK{\etalchar{+}}11]{goldberg11}
Andrew~V Goldberg, Sagi Hed, Haim Kaplan, Robert~E Tarjan, and Renato~F
  Werneck.
\newblock Maximum flows by incremental breadth-first search.
\newblock In {\em European Symposium on Algorithms}, pages 457--468. Springer,
  2011.

\bibitem[GHK{\etalchar{+}}15]{goldberg15}
Andrew~V Goldberg, Sagi Hed, Haim Kaplan, Pushmeet Kohli, Robert~E Tarjan, and
  Renato~F Werneck.
\newblock Faster and more dynamic maximum flow by incremental breadth-first
  search.
\newblock In {\em Algorithms-ESA 2015}, pages 619--630. Springer, 2015.

\bibitem[GLS81]{Grotschel1981}
Martin Gr{\"o}tschel, L{\'a}szl{\'o} Lov{\'a}sz, and Alexander Schrijver.
\newblock The ellipsoid method and its consequences in combinatorial
  optimization.
\newblock {\em Combinatorica}, 1(2):169--197, 1981.

\bibitem[GT88]{GoldbergTarjan88}
Andrew~V Goldberg and Robert~E Tarjan.
\newblock A new approach to the maximum-flow problem.
\newblock {\em Journal of the ACM (JACM)}, 35(4):921--940, 1988.

\bibitem[IFF01]{iwata2001}
Satoru Iwata, Lisa Fleischer, and Satoru Fujishige.
\newblock A combinatorial strongly polynomial algorithm for minimizing
  submodular functions.
\newblock {\em Journal of the ACM (JACM)}, 48(4):761--777, 2001.

\bibitem[IO09]{IwataO09}
Satoru Iwata and James~B Orlin.
\newblock A simple combinatorial algorithm for submodular function
  minimization.
\newblock In {\em ACM-SIAM Symposium on Discrete Algorithms (SODA)}, 2009.

\bibitem[Iwa03]{Iwata03}
Satoru Iwata.
\newblock A faster scaling algorithm for minimizing submodular functions.
\newblock {\em SIAM Journal on Computing}, 32(4):833--840, 2003.

\bibitem[JB11]{Jegelka11}
Stefanie Jegelka and Jeff~A Bilmes.
\newblock Online submodular minimization for combinatorial structures.
\newblock In {\em Proceedings of the 28th International Conference on Machine
  Learning (ICML-11)}, pages 345--352, 2011.

\bibitem[JBS13]{JegelkaBS13}
Stefanie Jegelka, Francis Bach, and Suvrit Sra.
\newblock Reflection methods for user-friendly submodular optimization.
\newblock In {\em Advances in Neural Information Processing Systems (NIPS)},
  2013.

\bibitem[JLB11]{jegelka11fast}
Stefanie Jegelka, Hui Lin, and Jeff~A Bilmes.
\newblock On fast approximate submodular minimization.
\newblock In {\em Advances in Neural Information Processing Systems}, pages
  460--468, 2011.

\bibitem[Kol12]{kolmogorov12}
Vladimir Kolmogorov.
\newblock Minimizing a sum of submodular functions.
\newblock {\em Discrete Applied Mathematics}, 160(15):2246--2258, 2012.

\bibitem[LSW15]{LeeSW15}
Yin~Tat Lee, Aaron Sidford, and Sam Chiu-wai Wong.
\newblock A faster cutting plane method and its implications for combinatorial
  and convex optimization.
\newblock In {\em IEEE Foundations of Computer Science (FOCS)}, 2015.

\bibitem[Nes12]{Nesterov10}
Yurii Nesterov.
\newblock Efficiency of coordinate descent methods on huge-scale optimization
  problems.
\newblock {\em SIAM Journal on Optimization}, 22(2):341--362, 2012.

\bibitem[NJJ14]{NishiharaJJ14}
Robert Nishihara, Stefanie Jegelka, and Michael~I Jordan.
\newblock On the convergence rate of decomposable submodular function
  minimization.
\newblock In {\em Advances in Neural Information Processing Systems (NIPS)},
  pages 640--648, 2014.

\bibitem[Orl09]{Orlin09}
James~B Orlin.
\newblock A faster strongly polynomial time algorithm for submodular function
  minimization.
\newblock {\em Mathematical Programming}, 118(2):237--251, 2009.

\bibitem[RKB04]{RotherKB04}
Carsten Rother, Vladimir Kolmogorov, and Andrew Blake.
\newblock Grabcut: Interactive foreground extraction using iterated graph cuts.
\newblock {\em ACM Transactions on Graphics (TOG)}, 23(3):309--314, 2004.

\bibitem[SAS16]{Shanu16}
Ishant Shanu, Chetan Arora, and Parag Singla.
\newblock Min norm point algorithm for higher order mrf-map inference.
\newblock In {\em Proceedings of the IEEE Conference on Computer Vision and
  Pattern Recognition}, pages 5365--5374, 2016.

\bibitem[Sch00]{Schrijver00}
Alexander Schrijver.
\newblock A combinatorial algorithm minimizing submodular functions in strongly
  polynomial time.
\newblock {\em Journal of Combinatorial Theory, Series B}, 80(2):346--355,
  2000.

\bibitem[Sch03]{Schrijver03}
A.~Schrijver.
\newblock {\em Combinatorial optimization - Polyhedra and Efficiency}.
\newblock Springer, 2003.

\bibitem[SK10]{StobbeK10}
Peter Stobbe and Andreas Krause.
\newblock Efficient minimization of decomposable submodular functions.
\newblock In {\em Advances in Neural Information Processing Systems (NIPS)},
  2010.

\bibitem[Wol76]{Wolfe76}
Philip Wolfe.
\newblock Finding the nearest point in a polytope.
\newblock {\em Mathematical Programming}, 11(1):128--149, 1976.

\end{thebibliography}

\end{document}